\documentclass[11pt]{article}
\usepackage[margin=1in]{geometry}
\usepackage[utf8]{inputenc} %
\usepackage[T1]{fontenc}    %
\usepackage[breaklinks]{hyperref}       %
\usepackage{url}            %
\usepackage{booktabs}       %
\usepackage{amsfonts}       %
\usepackage{nicefrac}       %
\usepackage{microtype}      %
\usepackage{xcolor}         %
\usepackage{natbib}
\usepackage{amsmath,amsthm,amssymb,mathtools,bbm}
\usepackage{thm-restate,thmtools}
\usepackage[ruled,vlined]{algorithm2e}
\usepackage{makecell}
\usepackage{upgreek}
\usepackage{bm}
\usepackage{bbm}
\usepackage[nameinlink]{cleveref}
\usepackage{amsfonts}
\usepackage{nicefrac}
\usepackage{microtype}
\usepackage{amsmath,amsfonts,amssymb,amsthm}
\usepackage{dsfont}
\usepackage{bm}
\usepackage{bbm}
\usepackage{array}
\usepackage{algorithmic}
\usepackage[inline, shortlabels]{enumitem}
\usepackage{tabulary,multirow}
\usepackage{tabularx}
\usepackage{booktabs}
\usepackage{multicol}
\usepackage{dsfont}
\usepackage{framed}

\def\eg{0}

\newtheorem{theorem}{Theorem}[section]

\newtheorem{lemma}[theorem]{Lemma}
\newtheorem{definition}[theorem]{Definition}

\newtheorem{corollary}[theorem]{Corollary}

\newtheorem{fact}[theorem]{Fact}

\newtheorem{assumption}[theorem]{Assumption}

\renewcommand{\epsilon}{\varepsilon}
\renewcommand{\hat}{\widehat}
\renewcommand{\tilde}{\widetilde}

\newcommand{\innerproduct}[2]{\left< #1, #2 \right>}
\newcommand{\inner}[2]{\left< #1, #2 \right>}

\newcommand{\norm}[1]{\left\lVert#1\right\rVert}

\newcommand{\para}[1]{\left(#1\right)}

\newcommand{\curlybrackets}[1]{\left\{#1\right\}}

\newcommand{\bset}[1]{\curlybrackets{#1}}

\newcommand{\absolutevalues}[1]{\left|#1\right|}
\newcommand{\abs}[1]{\left|#1\right|}

\DeclareMathOperator*{\Exp}{\mathbb{E}}

\newcommand{\EEs}[2]{\Exp_{#1}\left[#2\right]}

\renewcommand{\cite}[1]{\citep{#1}}

\newcommand{\R}{\mathbb{R}}

\newcommand{\reals}{\mathbb{R}}

\newcommand{\euclidianspace}{\mathcal{E}}

\DeclareMathOperator*{\argmax}{arg\,max}
\DeclareMathOperator*{\argmin}{arg\,min}

\newcommand{\asseq}{\coloneqq}

\newcommand{\vc}{\mathrm{VC}}

\newcommand{\errviF}[1]{{\mathop{\hbox{\rm Err}_{{#1}}}}}

\newcommand{\errvi}{{\mathop{\hbox{\rm Err}}}}

\newcommand{\cD}{\mathcal{D}}

\newcommand{\cX}{\mathcal{X}}

\newcommand{\cZ}{\mathcal{Z}}

\newcommand{\bw}{{\mathbf{w}}}

\newcommand{\bx}{{\mathbf{x}}}

\newcommand{\bz}{{\mathbf{z}}}

\def\va{{\bm{a}}}
\def\vb{{\bm{b}}}
\def\vc{{\bm{c}}}

\def\vu{{\bm{u}}}

\def\vx{{\bm{x}}}

\def\vz{{\bm{z}}}

\def\vzeta{{\bm{\zeta}}}

\title{Learning Variational Inequalities from Data: \\Fast Generalization Rates under Strong Monotonicity}

\author{
Eric Zhao%
\and 
Tatjana Chavdarova%
\and
Michael Jordan %
}

\date{University of California, Berkeley}

\begin{document}

\maketitle

\begin{abstract}
	Variational inequalities (VIs) are a broad class of optimization problems encompassing machine learning problems ranging from standard convex minimization to more complex scenarios like min-max optimization and computing the equilibria of multi-player games.
	In convex optimization, \emph{strong convexity} allows for fast statistical learning rates requiring only $\Theta(1/\epsilon)$ stochastic first-order oracle calls to find an $\epsilon$-optimal solution, rather than the standard $\Theta(1/\epsilon^2)$ calls.
    This note provides a simple overview of how one can similarly obtain fast $\Theta(1/\epsilon)$ rates for learning VIs that satisfy \emph{strong monotonicity}, a generalization of strong convexity.
    Specifically, we demonstrate how standard stability-based generalization arguments for convex minimization extend directly to VIs when the domain admits a small covering, or when the operator is integrable and suboptimality is measured by potential functions such as when finding equilibria in multi-player games.
\end{abstract}

\section{Introduction}\label{sec:introduction}
\textit{Variational inequalities}~\cite{stampacchia_formes_1964} model a broad range of optimization and fixed-point problems.  They are concerned with finding a solution $\vz^\star$ from a continuous domain $\mathcal{Z}$, such that
\begin{equation} \label{eq:vi} \tag{VI}
	\langle \vz-\vz^\star, F(\vz^\star) \rangle \geq 0, \quad \forall \vz \in \mathcal{Z} \,,
	\vspace{-.2em}
\end{equation}
where $F\colon \mathcal{Z}\mapsto \R^n$ is a continuous map, and $\mathcal{Z}$ is a subset of the Euclidean space $\R^d$.
The operator $F$ can be understood as encoding both the optimality condition and first-order information of a VI.

Variational inequalities (\ref{eq:vi}s) are a popular framework primarily due to their generality---they encompass convex minimization problems (when $F\equiv \nabla f$), min-max optimization problems, complementarity problems~\citep{cottle_dantzig1968complementary}, and equilibrium computation problems for general games.
VIs therefore provide a language for proving results that are applicable to a wide range of optimization problems.
As an example, consider a game with $k$ agents, each with a strategy $\vz_i\in \R^{d_i}$ and aiming to optimize their own objective $f_i \colon \R^d \to \R$.
Denoting the players' joint strategy with $\vz\equiv [\vz_1^\intercal, \dots, \vz_k^\intercal]^\intercal \in \R^d, d=\sum_{i=1}^k d_i$, we can write the problem of finding an equilibrium in this game as equivalently solving a~\ref{eq:vi} whose operator $F$ represents the gradients of each player's objective:
$$
F(\vz)\equiv
\begin{bmatrix}\nabla_{\vz_1} f_1(\vz)\\ 
	\dots                   \\ 
	\nabla_{\vz_k} f_k(\vz) \\
\end{bmatrix} \,.
$$
We refer the reader to \citet{facchinei_finite-dimensional_2003} for more examples.
Due to their generality, VIs also find applications in economics, game theory, and machine learning, to multi-agent reinforcement learning, GANs, and other problems involving trade-offs~\cite{Nagurney1995ProjectedDS,Martinez-legaz1988,gtvi2010,lowe2017multi,sidahmed2024vimarl,gidel2019vi}. 

In many applications of variational inequalities, the operator $F$ may not be known a priori. 
Instead, one may only have access to historical datapoints for the operator or other noisy finite estimates of operator evaluations.
Indeed, this is often the case in many practical applications of VIs:
\begin{enumerate*}[series=tobecont, label=(\roman*)]
	\item \textit{auction games}---where an auctioneer and bidders interact so as to allocate resources~\citep{mccafee1987,riley1989auctions,milgrom2004,krishna2009auction};
	\item various \textit{network games}~\citep{networkgames2010,parise_variational_2017}---which model the interaction between parts of information systems controlled by different parties; or 
	\item \textit{alternating multi-agent games}---where the goal is to perform well on an unseen game~\citep[e.g., mini-games in][]{vinyals2017starcraft}.
\end{enumerate*}
The key challenge with learning VI solutions in these applications is contending with uncertainty due to noisy observation model.
This noise can originate from different sources depending on the application, such as stochasticity in agents' decisions, noisy mini-batch evaluation of the operator, or having access only to unbiased variations of the ``ground-truth'' operator $F$.
One approach to learning VI solutions from noisy data is to run a conventional algorithm, such as gradient descent, on an empirical estimate of the operator and then bound generalization error, i.e., prove that a near-optimal solution for the empirical operator estimate is also near-optimal for the true operator.

One classical approach to bounding generalization error exploits the fact that learning algorithms that exhibit \emph{algorithmic stability}~\citep{Rogers1978AFS,devroye_wagner1979}---a measure of how sensitive an algorithm's outputs are to small perturbations in the input dataset---must have limited generalization error.
In strongly convex optimization problems, the key advantage of stability-based generalization analysis is that---even for simple first-order methods---they provide tighter sample complexity rates of order $O(1/\epsilon)$ where $\epsilon$ denotes the error tolerance~\cite{bousquet_stability_2002}.
This is in contrast to the $O(1/\epsilon^2)$ samples needed by stochastic approximation methods \cite{NemYud83}.
Recent papers have extended these stability-based generalization results to strongly-convex strongly-concave zero-sum games~\cite{li2022high,farnia_train_2020,zhang_generalization_2020}.

In this note, we demonstrate how to directly extend stability-based generalization proofs for convex minimization to variational inequalities.
In particular, we demonstrate how fast sample complexity rates of $O(1/\epsilon)$ are readily obtainable in strongly monotone VIs,
which are the natural VI generalization of strongly convex optimization (see Section~\ref{sec:preliminaries}) and include strongly-convex-concave min-max optimization as a special case.
A technical issue that we address is that the conventional measure of sub-optimality in VIs, the \emph{gap function}, can be unstable.
This is a property of gap functions that is not unique to VIs and can be bypassed in two ways.
First, if a VI's domain admits a small covering, the generalization error of the gap function is stable and fast generalization rates can be shown accordingly.
Second, if the operator is integrable and sub-optimality is measured in terms of potential functions, one can still prove fast generalization rates by analyzing the stability of a surrogate for the gap function; 
a common example is computing the Nash equilibrium of a multi-player game whose players have strongly convex/concave utilities.

\paragraph{Related works.}
Algorithmic stability is a standard approach to characterizing the generalization of a learning algorithm~\citep{mukherjee2006,shalev2010,pmlr-v80-charles18a,lei2020,Kuzborskij2017DataDependentSO}.
The work of~\citet{bousquet_stability_2002} derived high-probability generalization bounds, which was further improved on by~\citet{feldman_high_2019, klochkov_stability_2021}.
These results prove fast sample complexity rates of $O(1/\epsilon)$ for strongly convex minimization problems.
Recently, a number of works have proven similar fast rates for strongly-convex-concave games~\cite{li2022high,farnia_train_2020,zhang_generalization_2020}.
This note generalizes and streamlines these results by showing that, in fact, the same stability arguments used in convex minimization can be directly extended to VIs.

We also note that there is a long history of results studying the convergence of first-order methods in VIs.
In deterministic problems, the development of efficient algorithms with convergence guarantees has recently been the focus of interest in machine learning and optimization in both unconstrained~\citep[see, e.g.,][]{tseng_linear_1995,daskalakis2018training,mokhtari2019unified,mokhtari2020convergence,golowich2020last,azizian20tight,chavdarova2021hrdes,gorbunov2022extra,bot2022fast}, and constrained settings~\citep{NemYud83,BeckT03,nemirovski2004prox,cai2022constrVI,yang2022acvi}.
For stochastic problems, the seminal work of~\citet{juditsky_solving_2011} on stochastic mirror prox provides results on the stochastic approximation approach; these rates are slower in strongly monotone settings relative to ours, and cover stochastic approximation algorithms whereas our results cover \emph{batch} algorithms.

\section{Preliminaries}
\label{sec:preliminaries}

We will use bold small letters denote vectors 
and curly capital letters denote sets.
We also use $\cZ$ to denote a convex and compact set in the Euclidean space $\euclidianspace$, with inner product $\innerproduct{\cdot}{\cdot}$ and norm $\norm{\cdot}$.

\subsection{Variational inequalities (VIs)}
A~\ref{eq:vi} is represented by the tuple $(F, \cZ)$ containing its operator $F$ and domain $\cZ$.
We use $D$ to denote the diameter of the domain $\cZ$ where 
$D \asseq \max_{\vz, \vz' \in \cZ} \norm{\vz - \vz'}$ 
will be treated as a constant; for example, if $\cZ$ is the probability simplex $\Delta_d$ and $\norm{\cdot}$ is the $\ell_1$ norm, then $D = 2$.

The most common solution concept for a variational inequality is the \emph{Stampacchia solution}, which is defined as a solution $\vz_\star \in \cZ$ where $\innerproduct{F(\vz_\star)}{\vz_\star - \vz} \leq 0$ for all $\vz \in \cZ$.
A standard quantity for measuring how close a solution $\vz \in \cZ$ is to being a Stampacchia solution is the \emph{gap function}:
\begin{align*}
\errvi_F(\vz) \triangleq  \max_{\vu \in \cZ} \  \innerproduct{F(\vz)}{\vz - \vu} \,. 
\end{align*}
When $F$ is clear from context, we will omit the subscript and write $\errvi$.

When we have a VI with a conservative operator $F$---$F$ is the gradient(s) of a potential function(s)---another common numerical quantification of suboptimality is the gap between the best attainable potential versus the iterate's potential, which we will refer to as the \emph{potential gap}.
Specifically, if the domain has a product structure $\cZ = \cZ_1 \times \dots \times \cZ_k$, we can interpret the VI as a game with $k$ players.
For each player $i \in [d]$, the restriction of the operator $F$ to subspace $\cZ_i$, denoted by $F_i$, can be interpreted as the gradient of their potential/utility function $f_i$.
We can then define the suboptimality of a solution $\vz \in \cZ$ as sum of each player's suboptimality:
$$\mathrm{Err}_F(\vz) \triangleq  \sum_{i \in [k]} f_i(\vz) - \min_{\vz_i^* \in \cZ_i} f_i(\bz_{i}^*, \bz_{-i}).
$$

Throughout, we will adopt the standard assumptions that  an operator $F$ is $L$-\emph{Lipschitz} and \emph{monotone}.
These conditions ensure that a (weak) solution to the VI exists~\citep[Chapter $2$]{facchinei_finite-dimensional_2003}.
\begin{assumption}[$L$-Lipschitz operator]
	\label{assumption:smoothness}
	An operator $F\colon \cZ \rightarrow \euclidianspace$ is $L$-Lipschitz if and only if:
	\begin{align*}
		\norm{F(\vz) - F(\vz')} \leq L \norm{\vz - \vz'}, \quad 
		\forall \vz, \vz' \in \cZ \,.                           
	\end{align*}
\end{assumption}

\begin{assumption}[Monotone operator]
	\label{assumption:monotone}
	An operator $F\colon \cZ \rightarrow \euclidianspace$ is monotone if and only if:
	\begin{align*}
		\innerproduct{\vz - \vz'}{F(\vz) - F(\vz')} \geq 0, \quad 
		\forall \vz, \vz' \in \cZ \,.                             
	\end{align*}
\end{assumption}
Monotonicity can be understood as an analog of the assumption of convexity.
In some cases, we may assume the operator is strongly monotone, which is a stronger assumption as it implies monotonicity. One can understand strong monotonicity as generalizing strong convexity.
\begin{assumption}[$\mu$-strongly monotone operator]
	\label{assumption:stronglymonotone}
	An operator $F\colon \cZ \to \euclidianspace$ is said to be $\mu$-strongly monotone if and only if there exists $ \mu > 0$ such that:
	\begin{align*}
		\innerproduct{\vz - \vz'}{F(\vz) - F(\vz')} \geq \mu \norm{\vz - \vz'}^2,  \quad 
		\forall \vz, \vz' \in \cZ \,.                                                    
	\end{align*}
\end{assumption}
In strongly monotone settings, the ratio $ L/ \mu$ is known as the \emph{condition number} of the problem.
For ease of presentation, we will focus on how generalization rates scale with dimension, sample complexity, and error tolerance.
As such, the condition number should be understood to be a constant quantity on which we do not necessarily seek a tight characterization, although we will always state dependence on the condition number for completeness.

Finally, we will assume that the VIs we study are in an unconstrained setting.

\begin{assumption}[Unconstrained setting]
	There exists $\vz^* \in \cZ$ such that $F(\vz^*) = 0$.
\end{assumption}

\subsection{Generalization and algorithmic stability} \label{sec:alg_stability}

We are interested in solving a variational inequality problem without direct access to an operator oracle $\vz \mapsto F(\vz)$ by learning from data.
In this setting, we are restricted to only knowing the domain $\mathcal{Z}$ and having access to a noisy operator oracle $\Xi$ and a series of noise inputs $\vzeta_1, \dots, \vzeta_n$ drawn i.i.d. from some data distribution $\cD$ whose support we will denote as $\cX$.
Formally, $\Xi$ is a Borel function that given a sample $\vzeta \sim \cD$ provides the unbiased estimate
\begin{align*}
	\forall \vz \in \mathcal{Z}: \EEs{\vzeta \sim \cD}{\Xi(\vz, \vzeta)} = F(\vz). 
\end{align*}

Each $\vzeta_i$ can be understood as a datapoint in a dataset $X$ of size $n$.
Given a dataset $X \in \cX^n$, we will use the shorthand $\hat F_X(\vz) = \frac 1n \sum_{i=1}^n \Xi(\vz, \vzeta_i)$ and $\hat F_{X_i}(\vz) = \Xi(\vz, \vzeta_i)$.
Two datasets $X, X' \in \cX^n$ are \emph{neighboring} if they disagree in at most one entry.

We assume that noisy operator estimates satisfy the same assumptions as the true operator $F$ (i.e., Assumptions~\ref{assumption:monotone}, \ref{assumption:smoothness}), as is standard in generalization literature, e.g. \citet{hardt_train_2016,farnia_train_2020}.
We will also assume that operator samples are of bounded norm to control the variance of the noise.

\begin{assumption}[Bounded Operator Norm]
	\label{assumption:boundedoperator}
	The operator $F$ has bounded norm at every $\vz \in \mathcal{Z}$:
	\begin{align*}
		\norm{F(\vz)}^2 \leq K\,, \qquad \forall \vz \in \mathcal{Z}\,. 
	\end{align*}
\end{assumption}

An empirical variant of an optimization algorithm, e.g. gradient descent \eqref{eq:gd}, then follows from running the algorithms on the empirical operator $\hat F_X$ computed from a dataset $X \in \cX^n$ rather than the true operator $F(\vz)$.
We aim to understand the associated generalization error of using our empirical operator $\hat{F}_X$ instead of the true one, $F$.
Given a candidate solution $\vz \in \cZ$ produced by a learning algorithm trained on a dataset $\cX_n$, we define its empirical gap function value as
\[
	\hat \errvi_X (\vz) \asseq \max_{\bw \in \cZ} \inner{\hat F_X(\vz)}{\vz - \bw}
\]
and its generalization gap as $\errvi (\vz) - \hat \errvi_X (\vz)$.
In this work, we will mainly focus on settings where we drive the empirical gap function $\hat \errvi_X(\vz)$ to zero by training on the empirical operator $\hat F_X$ for many iterations, and only seek to control the generalization gap.
We analogously define the empirical potential gap $\hat{\mathrm{Err}}_X(\vz)$ and the generalization potential gap $\mathrm{Err}(\vz) - \hat{ \mathrm{Err}}_X(\vz)$. 

\paragraph{Stability.}
We say that a mapping is \textit{uniformly stable} if changing one record $\vzeta_i$ in its input (a dataset) does not significantly change the output.

\begin{definition}[Uniformly Stable Algorithm]
	An algorithm $A\colon \mathcal{X}^n \to \cZ$ is $\gamma$-uniformly stable if, for all datasets $X \in \mathcal{X}^n$ and $X' \in \mathcal{X}^n$ that differ in only one element it holds that:
	\begin{align*}
		\norm{A(X) - A(X')} \leq \gamma \,. 
	\end{align*}
\end{definition}
\noindent
Composing a uniformly stable algorithm with a Lipschitz loss results in a uniformly bounded function.
Uniformly bounded functions admit powerful generalization bounds~\citep{bousquet_stability_2002} that, for instance, can allow for obtaining dimension-free rates or fast rates.
\begin{definition}[Uniformly Bounded Functions]
	A mapping $f\colon \mathcal{X}^n \times \mathcal{X} \to [0, 1]$ is $\gamma$-uniformly bounded if, for all datasets $X \in \mathcal{X}^n$ and $X' \in \mathcal{X}^n$ that differ in only one element, and all evaluation data points $\vx \in \cX$, it holds that:
	\begin{align*}
		\absolutevalues{f(X, \vx) - f(X', \vx)} \leq \gamma \,. 
	\end{align*}
\end{definition}

\subsection{VI methods}\label{sec:methods}

Among the most commonly used first-order algorithms for solving~\ref{eq:vi}s is the (stochastic) \emph{gradient descent} method\if\eg1{ and the \emph{extragradient} methods~\citep{korpelevich_extragradient_1976}}\fi.
Gradient descent naturally extends to the setting of VI optimization as follows: 
\begin{equation} \tag{GD}\label{eq:gd}
	\vz_{t+1} = \vz_t - \eta F(\vz_t)  \,,
\end{equation}
where $\eta\in [0,1]$ denotes a step size.

\if\eg1{
The \emph{extragradient} (EG) method proposed by \citet{korpelevich_extragradient_1976} uses an extrapolated point $\vz_{t+\frac{1}{2}}$ 
using~\ref{eq:gd}: $\vz_{t+\frac{1}{2}} = \vz_t - \eta F(\vz_t)$, and then adds the gradient at the extrapolated point to the current iterate $\vz_t$:
\begin{equation} \tag{EG} \label{eq:extragradient}
	\begin{aligned}     
		\vz_{t+1} = \vz_t - \eta F(\vz_{t+\frac{1}{2}}) = \vz_t - \eta F\big(  \vz_t - \eta F(\vz_t) \big)  \,. 
	\end{aligned}
\end{equation}
Both~\eqref{eq:gd} and~\eqref{eq:extragradient} converge for strongly monotone VIs in deterministic (batch)~\citep{korpelevich_extragradient_1976,tseng_linear_1995} and stochastic settings~\citep{Hsieh2020explore,loizou2021stochastic,gorbunov2022stoch}.  
\eqref{eq:gd} is often used in practice due to its simplicity, and we study the popular~\eqref{eq:extragradient} as handling the extrapolation for sample complexity requires additional instructive steps.
The results we demonstrate for \eqref{eq:gd} and  \eqref{eq:extragradient} can also be applied to obtain corresponding sample complexity bounds for other first-order VI methods, such as \emph{optimistic} GD~\citep{popov1980}.
}
\else
{
The results we demonstrate for \eqref{eq:gd} can also be applied to obtain corresponding sample complexity bounds for other first-order VI methods, such as \emph{optimistic} GD~\citep{popov1980}.
}
\fi

\paragraph{Stability.}
It is known that the iterates produced by gradient descent are stable in norm in settings with smoothness (Assumption~\ref{assumption:smoothness}), strong monotonicity (Assumption~\ref{assumption:stronglymonotone}) and bounded operator norms (Assumption~\ref{assumption:boundedoperator}) \cite{hardt_train_2016}.
\begin{restatable}[GD stability]{lemma}{gdmonotone}
	\label{lemma:gdstronglymonotoneclaim}
	Consider a smooth strongly monotone variational inequality $(F, \mathcal{Z})$ (Assumption \ref{assumption:smoothness}, \ref{assumption:stronglymonotone}, \ref{assumption:boundedoperator}).
	Given neighboring datasets $X \in \cX_n$ and $X' \in \cX_n$, let $\vz_T$ and $\vz_T'$ denote the $T$th iterates of gradient descent (Algorithm \eqref{eq:gd}) on the empirical operators given by $X$ and $X'$ respectively. 
	For step sizes $0 < \eta < \frac{2 \mu}{L^2}$, uniform stability holds with
	\begin{align*}
		\norm{\vz_T - \vz_T'} \leq \frac{2K}{n(2 \mu - \eta L^2)}. 
	\end{align*}
\end{restatable}

\if\eg1
We can extend this result to argue the stability of extragradient descent.
A key step stability results of this kind is the argument that each step of the algorithm is contractive with respect to a fixed operator, which we establish for the extragradient algorithm.
\begin{restatable}[\ref{eq:extragradient} contractiveness for well-conditioned problems]{lemma}{egcontractive}
	\label{lemma:eg_contractive_str_monotone}
	Let $F\colon \cZ \to \reals^d$ be a $L$-Lipschitz  $\mu$-strongly monotone operator (Assumptions~\ref{assumption:smoothness}, \ref{assumption:stronglymonotone}), with $\mu > \frac{L}{2}$.
	Let $G_{\text{EG}}(\vz, F)$ denote the EGD update of an iterate $\vz$ given operator $F$, that is $G_{\text{EG}}\equiv I - \eta F \circ (I-\eta F) $. 
	For any learning rate $\eta > 0$, and $\forall \vz, \vz' \in \mathcal{Z}$ we have:
	\begin{align*}
		  \norm{G_{\text{EG}}(\vz, F) - G_{\text{EG}}(\vz', F)}^2 \leq  
		\underbrace{\para{ 2 - 2 \eta \mu + \eta^4  L^4 -  (2\eta\mu +1)(1-2\eta L + \eta^2 \mu^2)
		}}_{\triangleq c_{\text{EG}} (\eta)} \norm{\vz - \vz'}^2 \,.
	\end{align*}
	Hence, $G_{\text{EG}}(\cdot, F)$ is contractive for step sizes $\eta$ for which $ c_{\text{EG}}(\eta) < 1$, and $\{\eta_i \mid c_{\text{EG}} (\eta_i) < 1\}$ is non empty.  
\end{restatable}

The stability of the extragradient method then follows via a standard growth lemma argument.

\begin{restatable}[\ref{eq:extragradient} stability]{lemma}{egdmonotone}
	\label{lemma:egdstronglymonotoneclaim}
	Consider a smooth strongly monotone variational inequality $(F, \mathcal{Z})$ (Assumption \ref{assumption:smoothness}, \ref{assumption:stronglymonotone}, \ref{assumption:boundedoperator}).
	Given neighboring datasets $X \in \cX_n$ and $X' \in \cX_n$, let $\vz_T$ and $\vz_T'$ denote the $T$th iterates of extragradient descent (Algorithm \eqref{eq:extragradient}) on the empirical operators given by $X$ and $X'$ respectively. 
	For step sizes $\eta > 0$, uniform stability holds with
	\begin{align*}
		               & \norm{\vz_T - \vz_T'} \leq \frac{2K}{n(1 + \eta - \eta L^2(1 - \eta)(1 - 2 \eta \mu + \eta^2 L^2)) (1 - \eta L - (1 + \eta L))}. %
	\end{align*}
\end{restatable}
\fi

\section{Fast Gap Function Bounds for Small Domains}

The primary challenge with performing stability-based generalization analysis for variational inequalities is that the variational term for gap functions, $\argmax_{\vz^\star \in Z} \inner{F(\vz)}{\vz - \vz^\star}$,
is unstable.
This challenge is not unique to variational inequalities but rather an artifact of the gap function.
From another perspective, stability proofs involve a symmetrization step that typically assume that the training loss is a linear combination of losses computed on individual datapoints. In contrast, the term $\argmax_{\vz^\star \in Z} \langle \hat F_X(\vz), \vz - \vz^\star\rangle$ may depend non-linearly on the dataset $X$.

A natural approach is to take a union bound over the possible values of this variational term.
Specifically, we can define for every point in the solution set $w \in \cZ$, the loss
\begin{equation}
\label{eq:linearloss}
\inner{\hat F(\vz)}{\vz - w} = \tfrac 1n \sum_{i=1}^n \inner{\Xi(\vz, \vzeta)}{\vz-w},
\end{equation}
which is linear in the datapoints $\bset{\vzeta_i}_{i \in [n]}$.
This reduces the problem of bounding the generalization error of the gap function to bounding the generalization error of a linear loss function \eqref{eq:linearloss} and then taking a union bound over the possible values of $w$ in the domain $\cZ$.
Indeed, for strongly monotone VIs with small domains, this approach provides fast in-expectation generalization error bounds of $O(1/n)$, where $n$ is the size of the dataset.

\begin{restatable}{theorem}{uniform}
	\label{theorem:uniform}
	Let $(F, \cZ)$ be a smooth strongly monotone VI (Assumptions~\ref{assumption:smoothness}, \ref{assumption:stronglymonotone}, \ref{assumption:boundedoperator}).
	The expected generalization gap for a $\gamma$-stable algorithm $A\colon \cX^n \to \cZ$ given $n$ datapoints is, for $r > 0$,
	\begin{align*}
		  & \EEs{X \sim \cD^n}{\errviF{F}(A(X)) -  \errviF{\hat F} (A(X))}    \leq O\bigg(K r + (LD + K) \gamma \log(N(\cZ, r, \norm{\cdot}))\bigg), 
	\end{align*}
	where $N(\cZ, r, \norm{\cdot})$ is the $r$-covering number for $\cZ$.
\end{restatable}

\begin{proof}
    Let us first fix a choice of $\bw \in \cZ$.
    This stability proof begins with a standard symmetrization argument.
    Let us define $X^{(i)}$ as a copy of the dataset $X$ but whose $i$th record $X_i$ is replaced with that of dataset $X'$, i.e. $X'_i$.
    We can thus write
	\begin{equation}
		\label{eq:symmetrization}
		\begin{split}
			&      \EEs{X \sim \cD^n}{\inner{F(A(X))}{A(X) - \bw}}
			= \frac 1n \sum_{i=1}^n \EEs{\substack{X \sim \cD^n \\ X' \sim \cD^n}}{\inner{\hat F_{X_i}(A(X^{(i)}))}{A(X^{(i)}) - \bw}}.
		\end{split}
	\end{equation}
	Next, recall that, due to the $\gamma$-stability of the algorithm $A$, for any realization of the datasets $X$ and $X'$, we have by definition that 
	$
		\norm{A(X) - A(X^{(i)})} \leq \gamma
	$.
	Applying H\"older's inequality and the $L$-smoothness of $\hat F_{X_i}$, we can bound the symmetrization gap by
	\begin{equation*}
		\begin{split}
			\MoveEqLeft{ \inner{\hat F_{X_i}(A(X^{(i)}))}{A(X^{(i)}) - \bw}  - \inner{\hat F_{X_i}(A(X))}{A(X) - \bw}} \\
			& \leq  \inner{\hat F_{X_i}(A(X^{(i)})) - \hat F_{X_i}(A(X))}{A(X^{(i)}) - \bw}  + \inner{\hat F_{X_i}(A(X))}{A(X) - A(X^{(i)})} \\
			& \leq  \; L D \gamma + K \gamma.
		\end{split}
	\end{equation*}
    As a result, the random variable 
	\begin{align*}
		\Delta(X, X', \bw)
		\asseq \; & \frac 1n \sum_{i=1}^n \inner{\hat F_{X_i}(A(X))}{A(X) - \bw} - \inner{\hat F_{X_i}(A(X^{(i)}))}{A(X^{(i)}) - \bw},        
	\end{align*}
	is always bounded by $\Delta(X, X', \bw) \leq  2(LD + K) \gamma$.
	This implies that $\Delta(X, X', \bw)$ is a $(LD + K) \gamma$-sub-Gaussian random variable.
	For any finite subset $\cZ' \subset \cZ$ of our convex set, we can thus apply the maximal inequality to obtain a uniform bound over all potential values of the variational term $\bw$:
	\[
		\EEs{X, X' \sim \cD^n}{\max_{\bw \in \cZ'} \Delta(X, X', \bw)} 
		\leq \sqrt 2 (LD + K) \gamma \log(|\cZ'|).
	\]
	We can thus apply Jensen's inequality and \eqref{eq:symmetrization} to bound the generalization gap when $\bw$ is restricted to the subset $\cZ'$:
	\begin{equation}
		\label{eq:boundee}
		\sqrt 2 (LD + K) \gamma \log(|\cZ'|) 
		\geq \EEs{X \sim \cD^n}{\max_{\bw \in \cZ'} \inner{F(A(X))}{A(X) - \bw}} -  \EEs{X \sim \cD^n}{\errviF{\hat F}(A(X))}.
	\end{equation}
	We now choose the subset $\cZ'$ to cover the domain $\cZ$, proceeding with a one-step discretization for ease of exposition.
	Let $\cZ_\delta$ be the smallest $\delta$-covering of the set $\cZ$ in its norm $\norm{\cdot}$; i.e., for each $\vz \in \cZ$, there exists $\vz_\delta \in \cZ_\delta$ such that $\norm{\vz - \vz_\delta} \leq \delta$.
	Fixing $\delta > 0$, we can bound the generalization error by
	\begin{align*}
	\EEs{X \sim \cD^n}{\errvi_F(A(X)) - \errvi_{\hat F}(A(X))}    &= \EEs{X \sim \cD^n}{\max_{\bw \in \cZ_\delta} \inner{F(A(X))}{A(X) - \bw}}       - \EEs{X \sim \cD^n}{\errvi_{\hat F}(A(X))} \\
		     & + \EEs{X \sim \cD^n}{\errvi_{F}(A(X))}    - \EEs{X \sim \cD^n}{\max_{\bw \in \cZ_{\delta}} \inner{F(A(X))}{A(X) - \bw}}. 
	\end{align*}
	We can upper bound the first two summands by \eqref{eq:boundee} to be at most
	\begin{align*}
		        & \EEs{X \sim \cD^n}{\max_{\bw \in \cZ_{\Delta}} \inner{F(A(X))}{A(X) - \bw}}- \EEs{X \sim \cD^n}{\errvi_{\hat F}(A(X))}  \leq \sqrt 2 (LD + K) \gamma \log(|\cZ_\delta|).                                   
	\end{align*}
	As $\cZ_{\delta}$ is a $\delta$-cover, the next two summands are at most
	\begin{align*}
		        & \EEs{X \sim \cD^n}{\errvi_{F}(A(X))}      - \EEs{X \sim \cD^n}{\max_{\bw \in \cZ_{2^{-M}}} \inner{F(A(X))}{A(X) - \bw}} \leq K \delta.                                                                
	\end{align*}
\end{proof}

Due to the linear structure of the inner product in the gap function, it is possible to control the generalization error by taking a union bound only over the extremal points of the domain.
This means that, in certain settings, we can improve on the rates of Theorem~\ref{theorem:uniform}.
For example, let us consider the ($d+1$)-dimensional probability simplex $\Delta_d$ equipped with the L1 norm.
\begin{theorem}
    \label{theorem:uniformfast}
	Let $(F, \Delta_d)$ be a smooth strongly monotone variational inequality (Assumptions~\ref{assumption:smoothness}, \ref{assumption:stronglymonotone}, \ref{assumption:boundedoperator}).
 	The expected generalization gap for a $\gamma$-stable algorithm $A\colon \cX^n \to \Delta_d$ is:
    \begin{align*}
		  & \EEs{X \sim \cD^n}{\errviF{F}(A(X)) -  \errviF{\hat F} (A(X))}               \in O\big((L + K) \gamma \log d\big). 
	\end{align*}
\end{theorem}
\begin{proof}
Because linear functions are maximized on the vertices of a simplex, 
    we have that \[\min_{w \in \Delta_d} \inner{F(A(X))}{w} = \min_{i \in [d+1]} \inner{F(A(X))}{e_i},\]
    where $e_i$ is the $i$th standard basis vector.
    It thus suffices to only vary $w$ over the set $\cZ' = \bset{e_i}_{i \in [d+1]}$.
	The claim then follows by \eqref{eq:boundee}.
\end{proof}

We can, for example, plug the $\gamma = \frac{2K}{n(2\mu - \eta L^2)}$ stability of the gradient descent algorithm (Lemma~\ref{lemma:gdstronglymonotoneclaim}) into Theorem~\ref{theorem:uniformfast} to obtain a fast $O(1/n)$ generalization error bound in strongly monotone VIs.
\begin{corollary}
	Let $(F, \Delta_d)$ be a smooth strongly monotone variational inequality (Assumptions~\ref{assumption:smoothness}, \ref{assumption:stronglymonotone}, \ref{assumption:boundedoperator}).
 	The expected generalization error of the gradient descent algorithm $A\colon \cX^n \to \Delta_d$ with sufficiently small learning rate $\eta \to 0$ and large time horizon $T \to \infty$ is:
	\begin{align*}
		  & \EEs{X \sim \cD^n}{\errvi_F(A(X))} \leq \tilde O\bigg(\frac{K\cdot(L + K)}{n \mu}\bigg),
	\end{align*}
	where $\tilde O$ hides logarithmic factors.
\end{corollary}

\section{Fast Potential Gap Bounds for Games}
We now turn to the case of variational inequalities that describe multi-player games, where we assume that the operator $F$ is conservative and the domain has the product structure $\cZ = \cZ_1, \dots, \cZ_k$.
Throughout this section, we let $F_i$ denote the restriction of the operator $F$ to player $i$'s strategy space $\cZ_i$, and use \( \mu_i \) and \(L_i\) to denote player $i$'s strong monotonicity constant and the smoothness constant respectively; note that $\mu_i \geq \mu$ and $L_i \leq L$.

Another method of managing the instability of the gap function is to avoid bounding the generalization error of the gap function itself and only seek to bound the \emph{potential gap}.
In many applications, such as finding Nash equilibrium in multi-player games, only the latter is of interest and the gap function serves only as a means to upper bound the potential gap.
In these cases, we can define a surrogate gap function that replaces the unstable term $\argmax_{\vz^\star \in Z} \inner{F(\vz)}{\vz - \vz^\star}$ with
\begin{equation}\label{eq:modifiedvariationalterm}
	w^*(\vz) = \bigg[\argmin_{\vz_{i}^* \in \cZ_{i}} f_i(\vz_i^*, \vz_{-i}) \bigg]_{i \in [k]}.
\end{equation}
We denote this weak form of the gap function as
\begin{equation*}
    \mathrm{WeakGap}_F(\bz) \triangleq \inner{F(\vz)}{\vz - w^*(\bz)},
\end{equation*}
and note that it also upper bounds the potential gap by 
$\mathrm{Err}_F(\bz) \leq \mathrm{WeakGap}_F(\bz)$.
Unlike the gap function, this weaker gap function is stable due to the stability of the term $w^*(\bz)$, as noted by \citet{DBLP:conf/icml/LinJJ20} in the min-max settings.
We can also observe that the empirical form of this weak gap function, $
    \hat{\mathrm{WeakGap}}_X(\bz) = \inner{\hat F(\vz)}{\vz - w^*(\bz)}$,
 is linear in the dataset $X$ since the term $w^*(\bz)$ is fixed.
This allows us to reduce the problem of upper bounding the potential gap to the familiar problem of bounding the generalization error of an expected loss function.

\begin{restatable}{theorem}{gamebound}
	\label{theorem:gamebound}
	Let $(F, \cZ)$ be a smooth strongly monotone variational inequality (Assumptions~\ref{assumption:smoothness}, \ref{assumption:stronglymonotone}, \ref{assumption:boundedoperator}) corresponding to a $k$-player game.
	The expected generalization gap for a $\gamma$-stable algorithm $A\colon \cX^n \to \cZ$ is
	\begin{align*}
		  & \EEs{X \sim \cD^n}{\mathrm{WeakGap}_F(A(X)) -  \mathrm{WeakGap}_{\hat F} (A(X))}          \in O\bigg(\gamma (2DL + K \sum_{i \in [k]} \tfrac {L_i} {\mu_i} )\bigg). 
	\end{align*}
\end{restatable}

\begin{proof}

	The usual symmetrization argument implies that expected generalization gap is bounded by
	\begin{equation*}
		\begin{split}
			\EEs{X \sim \cD^n}{\mathrm{WeakGap}_F(A(X))}
			= & \; \frac 1n \sum_{j=1}^n \EEs{\substack{X \sim \cD^n \\ X' \sim \cD^n}}{\inner{\hat F_{X'_j}(A(X))}{A(X) - w^*(A(X))}} \\
			= & \; \frac 1n \sum_{j=1}^n \EEs{\substack{X \sim \cD^n \\ X' \sim \cD^n}}{\inner{\hat F_{X_j}(A(X^{(j)}))}{A(X^{(j)}) - w^*(A(X^{(j)}))}} \\
			\leq & \; \Delta + \EEs{X \sim \cD^n}{\mathrm{WeakGap}_{\hat F_X}(A(X))},
		\end{split}
	\end{equation*}
    where $\Delta$ is the quantity
	\begin{align*}
		\Delta =  & \inner{\hat F_{X_j}(A(X^{(j)}))}{A(X^{(j)}) - w^*(A(X^{(j)}))} - \mathrm{WeakGap}_{\hat F_X}(A(X)),                
	\end{align*}
	that we can decompose into the terms
	\begin{align*}
		(1) = & \inner{\hat F_{X_j}(A(X^{(j)}))}{A(X^{(j)})} - \inner{\hat F_{X_j}(A(X))}{A(X)}, 
	\end{align*}
	and
	\begin{align*}
		(2) = & \inner{\hat F_{X_j}(A(X))}{w^*(A(X))} - \inner{\hat F_{X_j}(A(X^{(j)}))}{w^*(A(X^{(j)}))}.
	\end{align*}
	We can control first term using H\"older's inequality and the smoothness of $\hat F_{X_j}$ to obtain
	\begin{align}\nonumber
		(1) & = \inner{\hat F_{X_i}(A(X^{(i)})) - \hat F_{X_i}(A(X))}{A(X^{(i)}))} + \inner{\hat F_{X_i}(A(X))}{A(X^{(i)}) - A(X)} \\
		    & \leq DL \gamma + K \gamma.   
			\label{eq:v1}                                                                                        
	\end{align}
	We can similarly bound the second term by
	\begin{align}
		\nonumber (2) & = \inner{\hat F_{X_j}(A(X)) - \hat F_{X_j}(A(X^{(j)}))}{w^*(A(X^{(j)}))} + \inner{\hat F_{X_j}(A(X))}{w^*(A(X)) - w^*(A(X^{(j)}))} \\
		    & \leq DL \gamma + K \norm{w^*(A(X)) - w^*(A(X^{(j)}))}. 
			\label{eq:v2}                                                                            
	\end{align}
	We now bound \( \norm{w^*(A(X)) - w^*(A(X^{(j)}))} \)
	in terms of $\norm{A(X) - A(X^{(j)})}$ using the technique of \citet{DBLP:conf/icml/LinJJ20}.
    This is the step that was not possible with the original gap function.
	We proceed by bounding each player's contribution to this norm separately.
	Due to the optimality of \( w^*(\vz)_i \), the restriction of $w^*(\vz)$ to the $i$th player's strategy space $\cZ_i$, the first-order optimality condition of the strongly convex function $f_i$ implies that, for all $\vx, \vx' \in \cZ_i$,
	$$\inner{F(w^*(\vz)_i, \vz_{-i})_i}{w^*(\vz)_i - \vx} \leq 0 \quad \mathrm{ and } \quad \inner{F(w^*(\vz')_i, \vz'_{-i})_i}{w^*(\vz')_i - \vx'} \leq 0.$$
	Choosing $\vx = w^*(\vz')_i$ and $\vx' = w^*(\vz)_i$ and summing the two inequalities gives
	$$\inner{F(w^*(\vz)_i, \vz_{-i})_i - F(w^*(\vz')_i, \vz'_{-i})_i}{w^*(\vz)_i - w^*(\vz')_i} \leq 0.$$
	In concert with the strong monotonicity of $F_i$, which implies 
	$$\inner{F(w^*(\vz)_i, \vz_{-i})_i - F(w^*(\vz')_i, \vz_{-i})_i}{w^*(\vz)_i - w^*(\vz')_i} - \mu_i \norm{w^*(\vz)_i - w^*(\vz')_i}^2 \geq 0,$$
	we have
	\begin{equation}
 \label{eq:variationalstability}
		\begin{split}
			\mu_i \norm{w^*(\vz)_i - w^*(\vz')_i}^2
			&\leq \inner{F(w^*(\vz')_i, \vz'_{-i})_i - F(w^*(\vz')_i, \vz_{-i})}{w^*(\vz)_i - w^*(\vz')_i} \\
			&\leq L_i \norm{\vz'_{-i}- \vz_{-i}} \norm{w^*(\vz)_i - w^*(\vz')_i}.
		\end{split}
	\end{equation}
	Choosing \( \vz = A(X) \) and \( \vz' = A(X^{(j)}) \) and subsituting the inequality into \eqref{eq:v2} gives
	\begin{align}
		\label{eq:v4}
	 (2)  \leq DL \gamma + K \sum_{i=1}^k \frac{L_i}{\mu_i} \norm{A(X)_i - A(X^{(j)})_i}. 
	\end{align}
	Summing \eqref{eq:v4} with \eqref{eq:v1} completes the proof.
\end{proof}

In this approach, we reduce the problem of bounding the potential gap to bounding the generalization error of an expected loss.
One implication is that we can directly apply---rather than reinventing---the many technical tools developed for the stability analysis of convex minimization.
For example, we can improve Theorem~\ref{theorem:gamebound} to the following high-probability guarantee due to \citet{klochkov_stability_2021}.

\begin{restatable}{theorem}{gameboundhp}
	\label{theorem:gameboundhp}
	Let $(F, \cZ)$ be a smooth strongly monotone variational inequality (Assumptions~\ref{assumption:smoothness}, \ref{assumption:stronglymonotone}, \ref{assumption:boundedoperator}) corresponding to a $k$-player game. 
    The expected generalization gap for a $\gamma$-stable algorithm $A\colon \cX^n \to \cZ$ when training error is small; i.e., $\mathrm{WeakGap}_{\hat F} (A(X)) \leq \epsilon$, is, with probability at least $1- \delta$,
	\begin{align*}
		  &\mathrm{WeakGap}_F(A(X))  \leq 2 \epsilon + O\Big(\log(\tfrac 1 \delta)\Big(\big(D L + K \sum_{i\in[k]} \tfrac{L_i}{\mu_i}\big) \gamma \log n + \tfrac{\big(D L + K \sum_{i\in[k]} \tfrac {L_i} {\mu_i }\big)^2}{n}\Big)\Big).
	\end{align*}
\end{restatable}

This provides a fast-rate high-probability potential gap bound of order $O(1/n)$ for gradient descent in games.
\begin{corollary}
	Let $(F, \cZ)$ be a smooth strongly monotone variational inequality (Assumptions~\ref{assumption:smoothness}, \ref{assumption:stronglymonotone}, \ref{assumption:boundedoperator}) corresponding to a $k$-player game.
 	Running the gradient descent algorithm $A\colon \cX^n \to \Delta_d$ with sufficiently small learning rate $\eta \to 0$ and large time horizon $T \to \infty$ results in an iterate $A(X)$ that with probability at least $1-\delta$ has a potential gap bounded by
	\begin{align*}
		  & \mathrm{Err}_F(A(X)) \leq \tilde O\Big(\tfrac {\log( 1 /\delta)}n \Big(\tfrac{K}{\mu} \big(D L + K \sum_{i\in[k]} \tfrac{L_i}{\mu_i}\big) + \big(D L + K \sum_{i\in[k]} \tfrac {L_i} {\mu_i }\big)^2\Big)\Big),
	\end{align*}
	where $\tilde O$ hides logarithmic factors.
\end{corollary}
\noindent

Theorem~\ref{theorem:gameboundhp} follows directly from proving that a Bernstein noise condition~\citep{koltchinskii2006,klochkov_stability_2021} is satisfied by the weak gap function.
\begin{definition}[Generalized Bernstein Noise Condition for Minimization]
	We say that an algorithm $A\colon \mathcal{X}_n \times \cX \to [0, M]$ satisfies the generalized Bernstein assumption for some parameter if, for any $B>0$, $\vz \in \mathcal{Z}$, there is $\vz^* \in \mathcal{Z}^*$ such that:
	\begin{align*}
		  & \EEs{x' \sim \cD}{\big(f(\vz, x')-f(\vz^*, x')}^2      \leq B\big(\EEs{x' \sim \cD}{f(\vz, x')} - \EEs{x' \sim \cD}{f(\vz^*, x')} \big) \,. 
	\end{align*}
	\label{def:bernstein}
\end{definition}

Indeed, this noise condition is satisfied in VIs when strong monotonicity holds. %
\begin{fact}
	\label{fact:bernstein}
	The weak gap function $\inner{F(\bz)}{\bz - w^*(\bz)}$ of a strongly monotone VI $(F, \cZ)$ satisfying the assumptions of Theorem~\ref{theorem:gameboundhp} fulfills the Bernstein condition with
	$B = (L D + K (1 + \sum_{i \in [k]} \tfrac {L_i} {\mu_i}))^2$.
\end{fact}
\begin{proof}
	We want to upper bound the variance of the weak gap function: 
	\begin{align*}
		\mathrm{GapVar} 
		&\triangleq \mathbb{E} \big[\big( \inner{\hat F(\vz)}{\vz - w^*(\vz)}    - \inner{\hat F(\vz^*, \vz)}{\vz^* - w^*(\vz)}\big)^2\big] \\
	 & = \mathbb{E}\Big[\Big(\inner{\hat F(\vz) - \hat F(\vz^*, \vz)}{\vz - w^*(\vz)} - \inner{\hat F(\vz^*, \vz)}{\vz^* - \vz - w^*(\vz^*) + w^*(\vz)}\Big)^2\Big].        
	\end{align*}
	To control the first summand, we use H\"older's inequality to bound
	\begin{align*}
		        & \inner{\hat F(\vz) - \hat F(\vz^*, \vz)}{\vz - w^*(\vz)}          \leq \norm{\hat F(\vz) - \hat F(\vz^*, \vz_{-i})}\norm{\vz - w^*(\vz)} \leq LD \norm{\vz - \vz^*}.                                            
	\end{align*}
	To control the second summand, we can again use H\"older's inequality to bound
	\begin{align*}
		        & \inner{\hat F(\vz^*, \vz)}{\vz^* - \vz - w^*(\vz^*) + w^*(\vz)} \leq K (\norm{\vz^* - \vz} + \norm{w^*(\vz) - w^*(\vz^*)})    \\ & \leq K (1 + \sum_{i \in [k]} \tfrac {L_i} {\mu_i})\norm{\vz^* - \vz}, 
	\end{align*}
	where the last step bounds the stability of the variational term surrogate $w^*$ in the same fashion as \eqref{eq:variationalstability}. 
	Summing these bounds, we have as desired
	\begin{align*}
		   \mathrm{GapVar}                                                                                              
		  & \leq  \mathbb{E}\Big[(L D + K (1 + \sum_{i \in [k]} \tfrac {L_i} {\mu_i}))^2\norm{\vz^* - \vz}^2\Big]             \\
		  & \leq   \mathbb{E}\Big[(L D + K (1 +  \sum_{i \in [k]} \tfrac {L_i} {\mu_i}))^2\inner{F(\vz)}{\vz - \vz^*}\Big]    \\
		  & \leq  \mathbb{E}\Big[(L D + K (1 +  \sum_{i \in [k]} \tfrac {L_i} {\mu_i}))^2\inner{F(\vz)}{\vz - w^*(\vz)}\Big]. 
	\end{align*}
\end{proof}

The high-probability bound then follows directly from \citet{klochkov_stability_2021} which provides a high-probability generalization error bound of order $O((\gamma \log(n) + B/n)\log(1/\delta))$ when the Bernstein condition holds with parameter $B$.
We defer the complete proof of Theorem~\ref{theorem:gameboundhp} to the Appendix.

\section{Discussion}
This note aims to provide an accessible exposition for deriving fast-rate generalization bounds in strongly monotone variational inequalities.
These derivations, which are direct extensions of stability arguments in convex optimization settings, demonstrate that the generality of working with VIs comes essentially ``for free'' when it comes to stability-based analyses of first-order methods.
There remain, however, gaps in our understanding of learning in strongly monotone VIs.
First, there seems to be an inherent hardness to minimizing gap functions with fast rates.
Our results on bounding the gap function relied on either assuming a small domain (Theorem~\ref{theorem:uniform}), which can introduce dimension dependence, or assuming domain structure (Theorem~\ref{theorem:uniformfast}).
Second, the generalization rates we derive are likely not optimal in their dependence on condition numbers and especially their dependence on $D$, the diameter of the domain.
We hypothesize that the latter dependence on $D$ may be inevitable when working with gap functions, but it is unclear if this should also be the case for potential gaps.

\section*{Acknowledgments}
This material is based upon work supported by the National Science Foundation Graduate Research Fellowship Program under Grant No. DGE 2146752, and by the European Union (ERC-2022-SYG-OCEAN-101071601).
Views and opinions expressed are however those of the authors only and do not
necessarily reflect those of the National Science Foundation, the European Union or the European Research Council
Executive Agency. Neither the European Union nor the granting authorities can be
held responsible for them.
The authors acknowledge support from the Swiss National Science Foundation (SNSF), grants 199740 and 214441.

\bibliographystyle{abbrvnat.bst}

\clearpage

\appendix

\section{Additional Preliminaries}\label{app:preliminaries}

\subsection{Additional definitions}

Contractive operators are defined as follows.
\begin{definition}[Contractiveness and non-expansiveness]\label{def:contractiveness}
	We say a mapping $G\colon \cZ\to \R^d$, with $\cZ \subseteq \R^d$, is non-expansive or $c$-contractive on $\cZ$ if there exists $c\in(0,1]$ or $c\in(0,1)$, respectively, such that:
	$$
	\norm{G(\vz) - G(\vz')} \leq c \norm{\vz-\vz'} \,, \quad \forall \vz,\vz' \in  \cZ \,.
	$$
\end{definition}

\section{Proofs} \label{app:proofs}

\subsection{Growth lemma and proof of Lemma~\ref{lemma:gdstronglymonotoneclaim}}
The growth lemma formalizes the notion of running a contractive iterative algorithm (as per Def.~\ref{def:contractiveness}) on two different but similar datasets.
\begin{lemma}[Generalized growth lemma \cite{farnia_train_2020} Lemma 2]
	\label{lemma:doubleseqgrowthlemma}
	Consider two sequences of updates, $\tilde{P}_1, \dots, \tilde{P}_T$ and $\tilde{P}_1', \dots, \tilde{P}_T'$, where each update is a mapping of form $\tilde{P}_i\colon \cZ \to \cZ$.
	Let $P\colon \cZ \to \cZ$ be some mapping, such that $P + \tilde{P}_1, \dots, P + \tilde{P}_T$ and $P + \tilde{P}_1', \dots, P + \tilde{P}_T'$ are two sequences of $\xi$-contractive mappings that each take $\cZ$ to $\cZ$.
	Fix some $\vz \in \cZ$ and define $\vz_1, \dots, \vz_T \in \cX$ to be the trajectory unrolled by applying the updates $P + \tilde{P}_1, \dots, P + \tilde{P}_T$ in sequence, starting at $\vz$.
	Analogously define $\vz_1', \dots, \vz_T' \in \cX$ with respect to $P + \tilde{P}_1', \dots, P + \tilde{P}_T'$.
	For timesteps $t$ where $\tilde{P}_t = \tilde{P}_t'$, we have $\norm{\vz_{t+1} - \vz_{t+1}'} \leq \xi \norm{\vz_{t} - \vz_{t}'}$.
	For all timesteps $t$, the following holds for every constant $r > 0$:
	\begin{align*}
		\norm{\vz_{t+1} - \vz_{t+1}'} \leq \xi \norm{\vz_{t} - \vz_{t}'}                                                               
		+ \sup_{\vz \in \mathcal{Z}} \norm{r \vz - \tilde{P}_t(\vz)} + \sup_{\vz \in \mathcal{Z}} \norm{r \vz - \tilde{P}_t'(\vz)} \,. 
	\end{align*}
\end{lemma}

\begin{proof}[Proof of Lemma~\ref{lemma:doubleseqgrowthlemma}]
	The first claim holds by definition of contractiveness.
	The second claim holds by fixing any $r > 0$ and observing that
	\begin{align*}
		\norm{\vz_{t+1} - \vz_{t+1}'}
		  & = \norm{P(\vz_t) - P(\vz_t') + \tilde{P}_t(\vz_t) - \tilde{P}_t'(\vz_t')}                                                                       \\
		  & = \norm{P(\vz_t) - P(\vz_t') + \tilde{P}_t(\vz_t) - \tilde{P}_t(\vz_t') + \tilde{P}_t(\vz_t') - \tilde{P}_t'(\vz_t')}                           \\
		  & \leq \norm{P(\vz_t) + \tilde{P}_t(\vz_t) - P(\vz_t') - \tilde{P}_t(\vz_t')} + \norm{\tilde{P}_t(\vz_t') - \tilde{P}_t'(\vz_t')}                 \\
		  & \leq \xi \norm{\vz_{t} - \vz_{t}'} + \norm{\tilde{P}_t(\vz_t') - r \vz_t'} +  \norm{r \vz_t' - \tilde{P}_t'(\vz_t')}                            \\
		  & \leq \xi \norm{\vz_{t} - \vz_{t}'} + \sup_{\vz \in \cZ} \norm{\tilde{P}_t(\vz) - r \vz} + \sup_{\vz \in \cZ}  \norm{r \vz - \tilde{P}_t'(\vz)}. 
	\end{align*}
	The first equality holds by definition of the iterate $\vz_{t+1} = \norm{P + \tilde{P}_t(\vz_t)}$ and likewise for $\vz_{t+1}'$.
	The first inequality is an application of the triangle inequality.
	The second inequality uses the $\xi$-contractiveness of $P + \tilde{P}_t$.
	The third inequality relaxes $\vz_t'$ to be a supremum over the whole domain $\cX$.
\end{proof}

For completeness, we retrace below a standard proof of the contractiveness of gradient descent on strongly convex functions but in the more general setting of strongly monotone VIs---observing that the proof goes through without modification.
\begin{lemma}[\ref{eq:gd} contractiveness]\label{lemma:stongly_monotone_and_lip}
	Let $F\colon \cZ \to \reals^d$ be a $\mu$-strongly monotone (Assumption~\ref{assumption:stronglymonotone}) and $L$-Lipschitz (Assumption~\ref{assumption:smoothness}) operator.
	For any learning rate $\gamma > 0$:
	\begin{align*}
		\norm{(\vz - \gamma F(\vz)) - \big(\vz' - \gamma F(\vz')\big)}^2_2 
		\leq (1- 2 \gamma \mu + \gamma^2 L^2) \norm{\vz - \vz'}^2_2,       
		\qquad \forall \vz, \vz' \in \mathcal{Z} \,.                       
	\end{align*}
	Moreover, if the learning rate satisfies $0 < \gamma < \frac{2\mu}{L^2}$:
	\begin{align*}
		\norm{\big(\vz - \gamma F(\vz)\big) - \big(\vz' - \gamma F(\vz')\big)}^2_2 
		\leq \para{1 - \frac{\mu^2}{L^2}} \norm{\vz - \vz'}^2_2,                   
		\qquad \forall \vz, \vz' \in \mathcal{Z} \,.                               
	\end{align*}
\end{lemma}

\begin{proof}[Proof of Lemma~\ref{lemma:stongly_monotone_and_lip}]
	Developing the left-hand-side of our first claim:
	\begin{align*}
		\norm{(\vz - \alpha F(\vz)) - \big(\vz' - \alpha F(\vz')\big)}^2_2 
		  & =  \norm{\vz - \vz'  - \alpha \big( F(\vz) - F(\vz')\big)}^2_2 \\ 
		  & = \norm{\vz - \vz'}^2_2                                        
		- 2\alpha \underbrace{ \langle F(\vz)-F(\vz') , \vz-\vz'\rangle }_{\geq \mu \norm{\vz-\vz'}^2_2}
		+ \alpha^2 \underbrace{ \norm{F(\vz)-F(\vz')}^2_2}_{ \leq L^2 \norm{\vz-\vz'}^2_2 } \\
		  & \leq (1 - 2\alpha \mu + \alpha^2L^2) \norm{\vz-\vz'}^2_2.      
	\end{align*}
	The second equality just expands the square, while the first inequality uses the $\mu$-strong monotonicity of $F$ to bound $\langle F(\vz)-F(\vz') , \vz-\vz'\rangle \geq \mu \norm{\vz-\vz'}^2_2$ and the $L$-Lipschitzness of $F$ to bound $\norm{F(\vz)-F(\vz')}^2_2 \leq L^2 \norm{\vz-\vz'}^2_2$. For given $\mu, L$ we can select $\alpha$ such that $1-2\alpha \mu + \alpha ^2 L^2 < 1\,.$
\end{proof}

We now proceed to bound the stability of the iterates of gradient descent.

\gdmonotone*
\begin{proof}[Proof of Lemma~\ref{lemma:gdstronglymonotoneclaim}]
	Let $j$ be the index of $\cX_n$ and $\cX_n'$ where the datasets disagree: $\xi_j \neq \xi_j'$.
	By definition of gradient descent, we can express each iterate as:
	\begin{align*}
		\vz_t = \vz_{t-1} - \frac{\gamma}{n} \sum_{i=1}^n \Xi(\vz_{t-1}, \xi_i), \qquad 
		\vz_t' = \vz_{t-1}' - \frac{\gamma}{n} \sum_{i=1}^n \Xi(\vz_{t-1}', \xi_i').    
	\end{align*}
	Thus, we can express $\vz_1, \dots, \vz_T$ as the trajectory of applying updates $P + \tilde{P}_1, \dots, P + \tilde{P}_T$ where we define:
	\begin{align*}
		P(\vz) = \vz - \frac{\gamma}{n} \sum_{i \in [n] \setminus \bset{j}} \Xi(\vz, \xi_i), \qquad 
		\tilde{P}_t(\vz) = \frac{\gamma}{n} \Xi(\vz, \xi_j).                                        
	\end{align*}
	Similarly, we can express $\vz_1', \dots, \vz_T'$ in terms of updates $P' + \tilde{P}_1', \dots, P' + \tilde{P}_T'$ where:
	\begin{align*}
		P'(\vz) = \vz - \frac{\gamma}{n} \sum_{i \in [n] \setminus \bset{j}} \Xi(\vz, \xi_i'), \qquad 
		\tilde{P}_t'(\vz) = \frac{\gamma}{n} \Xi(\vz, \xi_j').                                        
	\end{align*}
	Since $\xi_i = \xi_i'$ for all $i \neq j$, we have that $P = P'$.
	That is, the updates provided by each dataset disagree only in an additive component that corresponds to the disagreeing datapoint.
	The growth lemma, Lemma~\ref{lemma:doubleseqgrowthlemma}, thus implies:
	\begin{align*}
		\norm{\vz_{t+1} - \vz_{t+1}'} \leq \xi \norm{\vz_{t} - \vz_{t}'} 
		+  \norm{\tilde{P}_t(\vz_t')} + \norm{\tilde{P}_t'(\vz_t')},     
	\end{align*}
	where $\xi$ is the contractiveness of $P + \tilde{P}_t$.
	Since we are in a strongly monotone and smooth setting, Lemma~\ref{lemma:stongly_monotone_and_lip} bounds the contractiveness of gradient descent updates by $\xi \leq (1 - 2 \gamma \mu + \gamma^2 L^2)$.
	Because gradient descent has a bounded learning rate and we assume our operator is bounded, we have $\norm{\tilde{P}_t(\vz_t')} \leq \gamma K / n$ and $\norm{\tilde{P}_t'(\vz_t')} \leq \gamma K / n$.
	Thus:
	\begin{align*}
		\norm{\vz_{t + 1} - \vz_{t + 1}'}
		  & \leq (1 - 2 \gamma \mu + \gamma^2 L^2) \norm{\vz_{t} - \vz_{t}'} + \frac{2 \gamma K}{n}. 
	\end{align*}
	Recursing through $\vz_1, \dots, \vz_T$:
	\begin{align*}
		\norm{\vz_{T} - \vz_{T}'} \leq \frac{2 \gamma K}{n} \sum_{t=1}^T (1 - 2 \gamma \mu + \gamma^2 L^2)^{T-t} \leq \frac{2K}{n( 2  \mu - \gamma L^2)}\, 
	\end{align*}
	where the second inequality is because $\sum_{t=1}^T x^{T-t} = \frac{1 - x^T}{1 - x}$.
\end{proof}

\if\eg1
\subsection{Proof of Lemma~\ref{lemma:eg_contractive_str_monotone} and Lemma~\ref{lemma:egdstronglymonotoneclaim}}

In this section, we prove the contractiveness of extragradient descent updates in strongly monotone VIs and, by extension, the stability of its iterates in this setting.

We first recall Lemma~\ref{lemma:eg_contractive_str_monotone}. 

\egcontractive*

\begin{proof}[Proof of Lemma~\ref{lemma:eg_contractive_str_monotone}]

We first show the unconstrained case (in which case at the solution for a strongly monotone operator we have $F(\vz^\star) =\mathbf{0}$) for simplified notation and later the more general case (where $F(\vz^\star)$ is not necessarily $\mathbf{0}$ at $\vz^\star$).

\paragraph{Unconstrained domain, $F(\vz^\star)=\mathbf{0}$.}
Consider the unconstrained domain setting where $F(\vz^\star) = \mathbf{0}$.
Without loss of generality, suppose the solution is at $\mathbf{0}$, $\vz^\star=\mathbf{0}$.
Then, we need to show that the equilibrium distance of a subsequent iterate $\vz_{k+1}$  of \eqref{eq:extragradient} is smaller than that of the current iterate $\vz_{k+1}$, that is:  $ \norm{\vz_{k+1}-\vz^\star}^2\leq c \norm{\vz_{k}-\vz^\star}^2_2$, that is $\norm{\vz_{k+1}}^2_2 \leq c \norm{\vz_{k}}^2$, with $c<1$ for some choice of $\eta$.

Using that $\vz_{k+1} = \vz_k - \eta F\big(\vz_k-\eta F(\vz_k)\big)$ we get:
\begin{align}
\label{eq:eq9}
    \norm{\vz_{k+1}}^2_2 & = 
    \norm{\vz_k - \eta F\big(\vz_k-\eta F(\vz_k)\big) }^2_2 \nonumber\\
    & = 2 \langle \vz_k - \eta F(\vz_k), \vz_k - \eta F\big(\vz_k-\eta F(\vz_k)\big) \rangle %
    + \norm{\eta F(\vz_k) - \eta F\big( \vz_k - \eta F(\vz_k) \big)}_2^2 - \norm{\eta F(\vz_k) -\vz_k}_2^2 \nonumber\\
    &= 2 \Big(
            \norm{\vz_k}_2^2 \underbrace{ - \eta \langle F(\vz_k), \vz_k\rangle }_{\leq - \eta \mu \norm{\vz_k}_2^2 }
        \Big) 
        \underbrace{ - 2 \eta \langle
                    \vz_k - \eta F(\vz_k), F\big( \vz_k - \eta F(\vz_k)\big)
                  \rangle }_{\leq - 2 \eta \mu \norm{\vz_k-\eta F(\vz_k)}_2^2 }
                  \nonumber\\ 
        & \ \ \ + \underbrace{\eta^2  \norm{ F(\vz_k) - F\big( \vz_k - \eta F(\vz_k) \big) }_2^2 }_{\leq \eta^4 L^4 \norm{\vz_k}_2^2}
        - \norm{ \vz_k - \eta F(\vz_k)}_2^2 \,, \nonumber\\
        & \leq  \big(
                    2 - 2 \eta \mu + \eta^4 L^4
                \big) \norm{\vz_k}^2_2
                - (2\eta \mu + 1) \norm{ \vz_k - \eta F(\vz_k)}_2^2
\end{align}
where for the second line we used the identity $\norm{\va - \vc}^2_2 = 2 \langle \va-\vb, \va-\vc \rangle + \norm{\vb - \vc}^2_2 - \norm{\vb - \va}^2_2$, choosing $\va\equiv \vz_k, \vb = \eta F(\vz_k), \vc \equiv \eta F \big( \vz_k - \eta F(\vz_k) \big)$. The underscored inequalities use the assumptions of the lemma (Assumptions~\ref{assumption:smoothness}, \ref{assumption:stronglymonotone}) and that $\vz^\star = \mathbf{0}, F(\vz^\star) = \mathbf{0}$; that is $\norm{ F(\vz_k) -F(\vz^\star) }_2 \leq L \norm{\vz_k}_2$, and 
$\langle F(\vz_k) - F(\vz^\star), \vz_k - \vz^\star \rangle \geq \mu \norm{\vz_k}_2^2$.

By the reverse triangle inequality, we have:
\begin{align}
 \norm{ \vz_k - \eta F(\vz_k)}_2^2 &\geq \norm{\vz_k}^2_2 - 2\eta \underbrace{\langle \vz_k, \eta F(\vz_k)\rangle}_{\leq \norm{\vz_k}_2 \norm{F(\vz_k)}_2 \leq L  \norm{\vz_k}_2^2  } 
 + \underbrace{ \norm{ \eta F(\vz_k) }^2_2}_{\geq \eta^2\mu^2 \norm{\vz_k}^2_2 } \nonumber\\
 &\geq (1  - 2\eta L + \eta^2\mu^2) \norm{\vz_k}_2^2 \,,
\end{align}
where for the second term bound we used Cauchy–Schwartz and Lipschitzness consecutively. 
The bound on the third term follows by combining Cauchy–Schwartz $\langle F(\vz_k), \vz_k \rangle \leq \norm{F(\vz_k)} \norm{\vz_k}$ and strong monotonicity $\langle F(\vz_k), \vz_k \rangle \geq \mu \norm{\vz_k}^2$, yielding $\mu \norm{\vz_k}^2 \leq \norm{F{\vz_k}}\norm{\vz_k}$. 

Plugging the above in \eqref{eq:eq9} gives:
\begin{align}
    \norm{\vz_{k+1}}^2_2 & \leq \Big(
        2 - 2\eta \mu + \eta^4 L^4 - (1+2\eta\mu) ( 1- 2 \eta L + \eta^2 \mu^2 )
    \Big) \norm{\vz_k}^2_2 \,,
\end{align}
where the set of $\eta$ for which the coefficient above is smaller than $1$ is nonempty when $\mu > \frac{L}{2}$.

\end{proof}

\fi

\subsection{Proof of Theorem~\ref{theorem:gameboundhp}}
We restate the relevant result of \citet{klochkov_stability_2021} below for completeness.
\begin{theorem}[Theorem 1.1 of \citet{klochkov_stability_2021}]
	\label{theorem:klochkov_hp_stability}
	Let $f\colon \mathcal{X}^n \times \cX \to [0, M]$ be a data-dependent function that is $\gamma$-uniformly stable in its first argument, and that satisfies the Bernstein condition (Assumption~\ref{def:bernstein}) with parameter $B$.
	Suppose the optimization error of $f$, $\EEs{x \sim \cD}{f(\bx, x)}$, is at most $\Delta_{\mathrm{opt}}$.
	Then there is a universal constant such that, for any $\tilde c>0$ and with probability at least $1-\delta$ over the randomness of $\cX \sim \cD^n$,
	\begin{align*}
		  & \EEs{x' \sim \cD}{f(\cX, x')}-\inf_{\vz^* \in \mathcal{Z}} \EEs{x' \sim \cD}{f(\vz^*, x')}   \leq \Delta_{\mathrm{opt}}+ \tilde c \ \mathbb{E} [\Delta_{\mathrm{opt}}]   + c(1+\frac{1}{\tilde c}) (\gamma \log n+\frac{M+B}{n} ) \log (\frac{1}{\delta}) \,. 
	\end{align*}
\end{theorem}

We can now prove Theorem~\ref{theorem:gameboundhp}, which follows directly from the Bernstein condition bound of Fact~\ref{fact:bernstein} and Theorem~\ref{theorem:klochkov_hp_stability}.

\gameboundhp*
\begin{proof}
	We can first apply Theorem~\ref{theorem:klochkov_hp_stability} to the data-dependent function $f\colon \cX^n \times \cX \to [0, KD]$ where
	\[
		f(X, x') = \inner{\hat F_{x'}(A(X))}{A(X) - w^*(A(X))}
	\]
	and $w^*$ is as defined in \eqref{eq:modifiedvariationalterm}.
	Recall that by summing \eqref{eq:v4} with \eqref{eq:v1}, we have that $f$ is stable in its first argument:
	\[
		\abs{f(X, x') - f(X^{(j)}, x')}
		\leq  (2 D L + K) \gamma + K \gamma \sum_{i \in [k]} \tfrac {L_i} {\mu_i}
	\]
	for all $x' \in \cX$.
	Because $f$ also fulfills the Bernstein condition (Fact~\ref{fact:bernstein}), Theorem~\ref{theorem:klochkov_hp_stability} implies the main claim.
\end{proof}
\end{document}